\documentclass{plainThm}
\usepackage{setspace, fullpage, graphicx, url, verbatim, environ, booktabs, tabu}
\usepackage[lmargin=3.5cm,rmargin=3.5cm,marginpar=0pt,marginparsep=0pt]{geometry}
\usepackage{newtxtext, newtxmath}
\usepackage{pgffor}
\usepackage{color}
\usepackage{diagbox}
\usepackage[all]{xy}
\usepackage[mdyyyy]{datetime}
\usepackage{hyperref}
\hypersetup{
    colorlinks=true,
    linkcolor=blue,
    filecolor=magenta,      
    urlcolor=cyan,
}
\DeclareMathOperator{\cl}{Cl}

\DeclareMathOperator{\mv}{Mv}
\newcommand{\powerset}{\mathcal{P}}

\newcommand{\modelM}{\mathcal{M}}

\newcommand{\lr}[1]{\langle #1 \rangle}
\newcommand{\llrr}[1]{\{ #1 \}}
\newcommand{\ol}{\overline}
\newcommand{\sep}{\,|\,}

\newcommand{\Kf}{\ensuremath{K\!\!f}}
\newcommand{\Kv}{\ensuremath{K\!v}}

\newcommand{\BA}{\ensuremath{\mathbf{A}}}
\newcommand{\BG}{\ensuremath{\mathbf{G}}}
\newcommand{\BF}{\ensuremath{\mathbf{F}}}
\newcommand{\BP}{\ensuremath{\mathbf{P}}}
\newcommand{\BQ}{\ensuremath{\mathbf{Q}}}
\newcommand{\axiomLKVF}{\ensuremath{\mathbb{LKVF}}}
\newcommand{\logicLKVF}{\ensuremath{\mathbf{LKVF}}}

\newcommand{\LRA}{\Leftrightarrow}
\newcommand{\MP}{\ensuremath{\mathbb{MP}}}

\newcommand{\NEC}{\ensuremath{\mathbb{NEC}}}

\newcommand{\Gbkv}{\Kv_{\Gamma}}
\newcommand{\Gbk}{K_{\Gamma}}

\newcommand{\hs}{\mathfrak{h}}
\newcommand{\lat}{\mathfrak{L}}
\newcommand{\ttt}{\mathtt}
\newcommand{\Lmax}{\textrm{Lmax}}

\title{Epistemic Logic with Functional Dependency Operator}
\author{Yifeng Ding}
\affil{Group in Logic and the Methodology of Science, UC Berkeley.
  \thanks{  The author would like to give special thanks to
  Malvin Gattinger and Wesley H. Holliday
  for their unreserved helpful comments.}}
\affil[]{yf.ding@berkeley.edu}

\begin{document}

\maketitle
\abstract{
  Epistemic logic with non-standard knowledge operators,
  especially the ``knowing-value'' operator,
  has recently gathered much attention.
  With the ``knowing-value'' operator,
  we can express knowledge of individual variables,
  but not of the relations between them in general.
  In this paper,
  we propose a new operator \Kf{}
  to express knowledge of the functional dependencies between variables.
  The semantics of this $\Kf$ operator
  uses a function domain
  which imposes a constraint on what counts as a functional dependency relation.
  By adjusting this function domain,
  different interesting logics arise,
  and in this paper we axiomatize three such logics in a single agent setting.
  Then we show how these three logics can be unified
  by allowing the function domain to vary relative to different agents and possible worlds.
  A multiagent axiomatization is given in this case.}
\section{Introduction}

\textit{De re} knowledge or in general non-standard knowledge in epistemic logic is attracting continuing attention.
This line of research started from the very beginning of epistemic logic:
Hintikka discussed a ``knowing-who'' operator in \cite{Hintikka:kab},
and Plaza a ``knowing-value'' operator \Kv{} in his seminal work \cite{Plaza89:lopc}.
However, it is the recent effort in providing formal semantics and axiomatizations of those non-standard knowledge operators,
as outlined in the survey \cite{W16:survey},
that layed a solid foundation for further investigation.
Among all the non-standard knowledge operators axiomatized so far,
the ``knowing-value'',
or equivalently the ``knowing-what'' operator,
has received most attention,
partly due to its mathematical elegance and
partly because of its potential application in information security reasoning.
Recent major development of this \Kv{} operator started with the axiomatization in \cite{WF14, WF13},
followed by the simplification of the semantics in \cite{GW16} and
the enrichment of the language
through announcing values and propositions in \cite{baltag16:edl, EGW:kvp}.

Building on the above results about the ``knowing-value'' operator,
this paper considers the knowedge of the \textit{functional dependency} between variables,
which is a natural extension of
the knowledge of individual variables
to the knowledge of relations among variables.
The precise meaning of ``knowing a/the functional dependency between variables''
is not easy to pin down and might be context sensitive,
as illustrated by the difficulty to choose the correct article here:
it is safe to say ``knowing the value of a variable''
since a variable can only take one value in the actual world (or any world),
but there might be quite a lot functions,
different from each other,
yet all governing the relation between the same two variables in a set of possible worlds.
We postpone further discussion to the last section,
but it should be intuitive that
``functionality'' is at least a minimal requirement,
that is,
to know any functional dependency between variables $c$ and $d$,
at least for any two possible worlds where $c$ has the same value,
$d$ should also have the same value,
however different from the value of $c$.

Here one natural choice is
to make functionality
the only requirement of ``knowing a/the functional dependency between variables'',
and both \cite{EGW:kvp, baltag16:edl} made this choice.
The key intuition behind this choice is that,
what matters in the end are the values of variables.
Recall how implication in Heyting algebras for intuitionistic logic is defined:
$p\to q$ is the weakest proposition such that
if conjoined with $p$ by taking conjunction,
we get something stronger than $q$,
or in other words,
we are able to infer $q$.
In our knowing-value context,
we might also be interested and only interested in knowing the values.
Then,
functional dependency of $d$ upon $c$ should be interpreted
as the weakest proposition such that
if ``conjoined'' with the knowledge of the value of $d$,
we are able to infer the value of $c$.

The weakest proposition possessing this bridging-the-gap property
depends on how we interpret the word ``conjoin'' here.
If it is taken to be the propositional conjunction,
then what we get is again the propositional implication $\Kv(c) \to \Kv(d)$.
If ``conjoin'' means revealing the actual value of $c$ to the agent,
then $[c]\Kv(d)$ in \cite{EGW:kvp} is an exact formalization.
Model-theoretically speaking this means that
functionality between $c$ and $d$ holds on the set of possible worlds
where the value of $c$ is correct, and consequently,
once all possible worlds where $c$'s value is wrong are eliminated,
the value of $d$ becomes fixed and hence known.
If ``conjoin'' means to entertain the hypothesis that
one of the epistemically possible values of $c$ obtains,
then the functionality condition from $c$ to $d$
among all possible worlds is the minimal requirement.
This is equivalent to $K[c]\Kv(d)$,
which says:
I know that for all possible values that $c$ can take,
once that is revealed to be the real value of $c$,
the value of $d$ will also be known.
In \cite{baltag16:edl}, this is exactly the semantics of $K^cd$.

Another famous work on dependency
taking functionality as the only requirement is Dependence Logic \cite{Vaananen07:dpl, vaananen08:mdpl}.
The team semantics it uses for the dependence atom $=\!\!(c, d)$
is exactly the functionality condition,
though the teams in a model do not originate from an epistemic setting.

The semantics to be proposed in this paper
will differ from the above pure functionality approach and
will subsume it as a special case.
But the key inspiration comes from the basic strategy explained in \cite{W16:survey}:
pack an existential quantifier and a modal quantifier together in the form of $\exists x \Box\phi(x)$.
Under this pattern,
the knowledge of the functional dependency of variables $c $ and $ d$ is expressed as:
there exists a function $f$ in a predetermined function domain $\BF$ which works,
in the sense that $d = f(c)$,
in all epistemic scenarios.
Thus, $\BF$ can be seen as an agent's prior knowledge about possible functional dependency relations,
and to know the dependency between variables is
to find a possible function that works or explains all possibilities.
To put it more colloquially,
to know the functional dependency between $c$ and $d$
is not simply to see that functionality holds between them,
but also to see that the functional relation ``make sense''.
Let us use $\Kf(C, d)$
to express this knowledge of functional dependency of $d$ upon a finite set of variables $C$. 

As argued above,
when ``knowing-dependency'' serves as a tool for expressing potential ``knowing-value'',
we do not need a requirement stronger than functionality.
But this is not always the case.
Consider a typical scenario in information security:
agent $\mathtt{A}$ receives an encrypted message $d = \texttt{enc}(c)$ from agent $\mathtt{B}$.
Ideally, $\mathtt{A}$ knows the value of $d$,
say $d = 0$,
but knows nothing about $c$.
So the epistemically possible worlds for $\mathtt{A}$ are
\[\llrr{c = 0, d = 0}, \llrr{c = 1, d = 0}, \llrr{c = 2, d = 0}, \cdots .\]
Certainly the functionality from $c$ to $d$ holds
as $d$ has only one possible value.
But agent $\mathtt{A}$ is apparently ignorant about
the functional relationship between variables $d$ and $c$.
The witness to the functionality here is the constant function $0$,
which is extremely unlikely to be the encryption function $\texttt{enc}$ that $\mathtt{B}$ uses.
So agent $\mathtt{A}$ would not in this case
assert that she knows that the message $d$ she receives
is derived from the message $c$ that $\mathtt{B}$ intends to send through some encryption:
no encryption function she deems possible would allow all those possibilities.
Thus, to claim the knowledge of the functional dependency of $d$ on $c$,
we do need something more than functionality.
With our operator $\Kf$,
we can use $\Kf_{\mathtt{A}}(\llrr{c}, d)$ to express
``$\mathtt{A}$ knows a functional dependency relation between $c$ and $d$
that is plausible in the information security context'',
if we let $\BF$ to be the set of all functions that is plausible in this context.

Thus, the $\Kf$ operator can be used to model scenarios where
the value of variables in the realized world (the agent's world) is not the sole concern of the agent.
It might be that our agent does not want an inexplicable relationship between variables,
or it might be that the agent requires that any functional dependency she knows to be applicable
not only to her actual world but also to worlds metaphysically possible or worlds evolved in time,
where some \textit{a priori} rules preclude too strange functional dependency relationships.
In the previous case,
certainly $d$ is known to $\mathtt{A}$ already,
but the constant function that witnesses the functionality there
is not likely to be applicable to another round of message exchange. 

In the rest of the paper,
we first define the logic that incorporates
knowledge $K$, ``knowing-value''$\Kv$ and ``knowing-function''$\Kf$ operators
which we call $\logicLKVF$ and the corresponding base axiom system $\axiomLKVF$.
Then we show how different domains of functions,
viewed as a parameter of $\logicLKVF$,
induce different sets of validities and axioms.
Then all those cases will be put into a unified framework
where a multiagent logic with the same operators is axiomatized.
In the last section,
we will discuss
further interpretations of
``knowing a/the functional dependency between variables'' and
possible future work.

\section{Preliminaries}

\subsection{Syntax and Semantics of \logicLKVF}

\begin{definition}[Syntax]\label{syntax}
    Given
    a countably infinite set $\BP$ of propositional letters and
    a set $\BQ$ of the names of variables,
    the formulas in \logicLKVF{} are defined by:
    $$\phi ::= \top\ |\ p\ |\ \Kv(d)\ |\ \Kf(C, d)\ |\ (\phi\land\phi)\ |\ \lnot\phi\ |\ K\phi$$
    where $p \in \BP,
    d \in \BQ$, and $C\subseteq_{fin}\BQ$.
    $\subseteq_{fin}$ means a finite subset,
    possibly empty.
\end{definition}
Here $\Kv(d)$ is to be interpreted as ``knowing the value of $d$'',
and $K\phi$ ``knowing that $\phi$ is the case''.
$\Kf(C, d)$ says that the agent knows a functional dependency relationship from $C$ to $d$.
By convention,
we set $\bot, (\phi\lor\psi), (\phi\to\psi)$ as
$\lnot\top, \lnot(\lnot\phi \land \lnot\psi), \lnot(\phi \land \lnot\psi)$,
and omit unnecessary parentheses.
We also write $\Kf(c, d)$ as an abbreviation of $\Kf(\llrr{c}, d)$. 

In order to interpret the $\Kf$ operator in \logicLKVF,
we need a predefined domain \BG{} of possible values for variables in \BQ,
and a set of \BF{} functions on this \BG{}.
\BF{} might contain polyadic functions in $\BG^{\BG^{n}}$ and also zero-adic functions.
Formally $\BF\subseteq \bigcup_{n = 0}^{\infty}\BG^{\BG^n}$.
It is important to note here that in this setting,
\BF{} and \BG{} are important parameters of $\logicLKVF$ instead of parts of the models.
In other words, they are shared by all models in the logic.

As we are considering single agent S5,
no explicit accessibility relation is needed.
So formally, a model is:
$$
\modelM = \lr{W,  U,  V}
$$
where $W$ is the set of possible worlds,
$U:W\times\BP\to \llrr{0,1}$ is the assignment for propositional letters,
and $V:W\times\BQ\to \BG$ is the assignment for variables.
For any finite subset $C$ of $\BQ$,
we fix an order of the elements in $C$ and define $V(w, C) = \lr{V(w, d)\ |\ d\in C}$.
When $C$ is empty,
this degenerates into the unique empty tuple.
We call this the joint assignment of variables in $C$,
and whenever we have a function from $\BQ$ to $\BG$,
if it is applied to a set $C$,
we mean this joint assignment.
Now the truth conditions are:

\begin{definition}[Semantics]
    $$
    \begin{array}{lll}
      \modelM, w \vDash \top & & \textrm{always} \\
      \modelM, w \vDash p &\Leftrightarrow & U(w, p) = 1 \\
      \modelM, w \vDash \Kv(d) & \LRA & \exists x\in \BG,  \forall w'\in W,  V(w',  d) = x \\
      \modelM, w \vDash \Kf(C, d) & \LRA & \exists f\in \BF,  \forall w'\in W,  V(w', d) = f(V(w', C)) \\
      \modelM, w \vDash \phi\land\psi & \LRA & \modelM, w\vDash \phi\ \textrm{and}\ \modelM, w\vDash \psi \\
      \modelM, w \vDash \lnot\phi &\LRA &\textrm{not}\ \modelM, w\vDash \phi \\
      \modelM, w \vDash K\phi &\LRA & \forall w'\in W,  \modelM, w'\vDash \phi
    \end{array}
    $$
\end{definition}
Here the $\Kv$ operator has the same meaning as that of $\Kv$ in \cite{WF14}:
$\Kv(d)$ means that under current epistemic uncertainty,
the value of $d$ is certain.
The new operator $\Kf(C,d)$ here means:
the agent can find a function in the set of available functions $\BF$
that can be used to explain the functional dependency relation between $C$ and $d$.
While both operators have the same structure in their semantics,
namely $\exists\Box$,
the key difference here is that,
if $\Kv(d)$ is true,
only one value will be the witness,
yet for $\Kf$ this is usually not the case. 

To summarize,
our logic $\logicLKVF$
extends
the standard propositional epistemic logic
by
adding $\Kv(d)$ and $\Kf(C, d)$ to the language,
adding a valuation of the variables to the models,
and introducing a new function domain $\BF$ as part of the logic.
Now it has the following parameters:
\begin{itemize}
    \item $\BP$: the set of propositional letters
    \item $\BQ$: the set of variable names
    \item $\BG$: the set of values that variables can take
    \item $\BF$: the set of functions that the agent deems possible \textit{a priori}.
\end{itemize}
All of them will have some effect on the validities of $\logicLKVF$,
but $\BP$ and $\BQ$ will remain unchanged throughout the whole paper,
since they can be viewed as part of the language.
$\BG$ needs to be large for completeness results,
and we will specify how large it should be.
$\BF$ will change the validities in $\logicLKVF$ in an interesting way.
Thus, it will be one of the main focuses of this paper.
Later we show how $\BF$ can also be put into the models.

\subsection{Base Axiom System and Soundness Condition}
As defined above,
the $\Kf$ operator expresses functional dependencies among variables
and thus resembles the dependency relation in database theory.
Using Armstrong's three axioms in \cite{Armstrong74},
we obtain this base system \axiomLKVF:\\
\begin{center}
    \begin{tabular}{lclc}
      $\mathtt{TAUT}$ & \textrm{Propositional Tautologies} & & \\
      $\mathtt{K}   $ & $K(\phi\to\psi)\to (K\phi\to K\psi)$  & $\mathtt{KV4} $ & $\Kv(d)\to K\Kv(d)$ \\
      $\mathtt{T} $   & $K\phi\to\phi$                        & $\mathtt{KV5} $ & $\lnot\Kv(d) \to K\lnot\Kv(d)$\\
      $\mathtt{4} $   & $K\phi\to KK\phi$                     & $\mathtt{KF4} $ & $\Kf(C, d)\to K\Kf(C, d) $\\
      $\mathtt{5} $   & $\lnot K\phi \to K\lnot K\phi$        & $\mathtt{KF5} $ & $\lnot\Kf(C, d) \to K\lnot\Kf(C, d)$ \\
    \end{tabular}
    
\end{center}
\begin{center}
    \begin{align*}
      \mathtt{PROJ}  &\phantom{{ } }\quad  { \Kf(C, c)\quad c\in C } \\
      \mathtt{TRAN}  &\phantom{{ } }\quad  { \left(\bigwedge_{d\in D}\Kf(C, d)\right) \land \Kf(D, e)\to \Kf(C, e)}\\
      \mathtt{VF} & \phantom{{ } }\quad  { \left(\bigwedge_{c\in C}\Kv(c)\right) \land \Kf(C, d)\to \Kv(d) }
    \end{align*}
\end{center}

\begin{center}
    \begin{tabular}{lclc}
      \MP& $\dfrac{\phi, \phi\to\psi}{\psi}$ & \NEC &$\dfrac{\phi}{K\phi}$\\
    \end{tabular}
\end{center}

Here only the projectivity and transitivity axioms are used.
The reason is that in our language
the syntax of $\Kf$ allows only one variable to be dependent upon a set of variables,
not a set upon a set.
Thus, the additivity property $\Kf(A, B)\land \Kf(A, C) \to \Kf(A, B\cup C)$
dealing with the second set of variables after $\Kf$
is not used and
will follow from the properties of the conjunction
if we define $\Kf(C, D)$ to be $\bigwedge_{d\in D}\Kf(C, d)$.
Then the augmentation axiom in the usual presentation of Armstrong's axioms
follows from additivity, projectivity, and transitivity.
To show this, suppose $\Kf(A, B)$.
By projectivity,
$\Kf(A\cup C, C)$ and $\Kf(A\cup C, A)$.
Together with the assumption $\Kf(A, B)$,
we have $\Kf(A\cup C, B)$.
So by additivity applied to $\Kf(A\cup C, C)$ and $\Kf(A\cup C, B)$,
$\Kf(A\cup C, B\cup C)$.

By convention,
an empty conjunction is $\top$.
So when the set $D$ in $\ttt{TRAN}$ is empty,
it actually says $\Kf(\emptyset, e) \to \Kf(C, e)$ for all $C\subseteq \BQ$.
And when the set $C$ in $\ttt{VF}$ is empty,
it says $\Kf(\emptyset, d) \to \Kv(d)$.

We will discuss the axiomatizations of three different settings
using a large, a small, and an intermediate $\BF$ in $\logicLKVF$ respectively.
For them,
we either use $\axiomLKVF$ itself or add some other special axioms.
To simplify repetitive work,
here we give a condition on $\BF$ in $\logicLKVF$ for the soundness of $\axiomLKVF$:
\begin{proposition}\label{SoundCondition}
    When $\BF$ satisfies the following,
    $\axiomLKVF$ is sound with respect to $\logicLKVF$:
    \begin{itemize}
        \item For every $i,j\in\mathbb{N}$ such that $0 <i \le j$ and function $f: \BG^j \to \BG,
        f(x_1,x_2,\ldots,x_j) = x_i$ is in $\BF$.
        We denote this special projection function as $id_{i,j}$.
        \item For every $f\in \BF$,
        if $f$ is $n$-ary with $n \ge 1$,
        then for every $g_1,\ldots g_n\in \BF$,
        $f(g_1(),\ldots,g_n())\in \BF$.
        Namely, $\BF$ is closed under function composition.
    \end{itemize}
\end{proposition}
\begin{proof}
    Here we only prove the soundness of the three less trivial axioms:
    \begin{itemize}
        \item By the first property of $\BF$,
        $\mathtt{PROJ}$ holds.
        If $d\in C$,
        suppose $d$ appears in $C$ as the $i$th variable,
        then $V(w,d) = V(w,C)[i]$ always holds,
        and thus the witness of $\Kf(C, d)$ is $id_{i,|C|}$.
        \item By the second property of $\BF$, $\mathtt{TRAN}$ holds.
        The antecedent of this axiom states the existence of $f$ and $g_i$s in the second property.
        So the composition of $f$ and $g_i$s exists in $\BF$,
        which witnesses the consequent of $\mathtt{TRAN}$.
        \item We want to show
        \[\mathtt{VF}: \left(\bigwedge_{c\in C}\Kv(c)\right) \land \Kf(C, d)\to \Kv(d).\]
        Let $C$ be enumerated as $c_1,\ldots,c_n$ and suppose the antecedent in $\mathtt{VF}$ holds.
        Then $\bigwedge_{c\in C}\Kv(c)$ is true.
        This means we have a tuple $\bar{a}\in\BG^n$ such that
        $$\forall w, V(w, C) = \bar{a}.$$
        Further we have $\Kf(C,d)$,
        which means we have a $f\in \BF$ such that
        $$\forall w, V(w, d) = f(V(w, C)) = f(\bar{a}). $$
        Thus, there exists an element $b := f(\bar{a}) \in \BG$ such that
        $d$ evaluates to it in all possible worlds. \qed
    \end{itemize} 
\end{proof}
We will briefly mention how
$\BF$ is going to satisfy this soundness condition in all the following cases.

\section{Full Domain of Functions}

In this section,
we deal with the case where $\BF$ is as large as possible,
namely $\BF = \bigcup\llrr{\BG^{\BG^i}\ |\ i\in \mathbf{N}}$.
Now the $\Kf$ operator degenerates into a functionality test,
as all functions are allowed:
\begin{multline*}
    \modelM, w\vDash \Kf(C, d) \LRA \\
    \forall w_1,  w_2\in W,  V(w_1,  C) = V(w_2,  C)
    \Rightarrow V(w_1, d) = V(w_2,  d).$$
\end{multline*}
This is true because once we have the right hand side true,
we will obtain a partial function $f$ satisfying $\forall w\in W,  f(V(w, C)) = V(w, d)$.
And it is trivial to extend this partial function into a total function.

Now, if $\modelM, w\vDash \Kv(d)$,
then $\forall w_1,  w_2\in W,  V(w_1, d) = V(w_2, d)$,
so the right hand side of the above truth condition holds,
and consequently,
$\Kf(C,d)$ is true in $\modelM, w$.
This justifies the soundness of our new axiom in this case:
$$\mathtt{EXT}:\quad \Kv(d)\to\Kf(C, d)$$
where $C\subseteq_{fin}\BQ$, possibly empty.
We name this axiom $\mathtt{EXT}$ because it means that in this case
every function on $\BG$, regardless of its meaning,
can serve as a witness of the truth condition of $\Kf$.
Further,
$\BF$ satisfies the condition given in Proposition \ref{SoundCondition},
so $\axiomLKVF+\mathtt{EXT}$ is sound.
In the following,
we prove that if $\BG$ is sufficiently large,
then $\axiomLKVF+\mathtt{EXT}$ is in fact complete as well.

Given an arbitrary set $A$ of formulas consistent in $\axiomLKVF+\mathtt{EXT}$,
the Lindenbaum lemma enables us to construct
a maximal consistent set $\Gamma$ such that
$A\subseteq\Gamma$.
Now to build a model for $\Gamma$,
we need to accompany this $\Gamma$ by other maximal consistent sets (possible worlds).
For example,
if we have $\lnot \Kf(C,d)$ in $\Gamma$,
then we need two possible worlds on which
the values of $C$ coincide while the values of $d$ on them diverge.

To this end,
we first define some useful sets.
Given any maximal consistent set $\Gamma$, define
$$\Gbk = \llrr{\phi \sep K\phi\in\Gamma},  \Gbkv = \llrr{d\sep \Kv(d)\in\Gamma}.$$
They collect all the propositional and the value knowledge respectively in $\Gamma$.
For any $C\subseteq\BQ$,
we say $C$ is \textbf{closed under} $\Kf$ \textbf{in} $\Gamma$ if
for all $C_f\subseteq_{fin}C$ and $d\in\BQ$ such that $\Kf(C_f, d)\in\Gamma$,
we have $d\in C$ as well.
Using axioms $\mathtt{TRAN}$ and $\mathtt{PROJ}$,
it is not hard to see that for all $C\subseteq_{fin}\BQ$,
$$C^{+\Gamma} := \llrr{d \sep \Kf(C, d)\in\Gamma}$$
is closed under $\Kf$ in $\Gamma$ and $C\subseteq C^{+\Gamma}$.
This can be seen as the dependency hull of the finite set $C$.
An important observation is that, by axiom $\ttt{VF}$,
if $\Kf(\emptyset, d)\in \Gamma$,
then $\Kv(d)\in\Gamma$,
so $\emptyset^{+\Gamma} \subseteq \Gbkv$.
Also,
by axiom $\ttt{EXT}$,
if $\Kv(d)\in\Gamma$ then $\Kf(C, d)\in\Gamma$ for all $C\subseteq \BQ$.
So $\Gbkv \subseteq C^{+\Gamma}$ for all $C\subseteq_{fin} \BQ$,
and in particular $\Gbkv \subseteq \emptyset^{+\Gamma}$.
So $\Gbkv = \emptyset^{+\Gamma}$.
This motivates us to define the set of all finitely generated closed sets:
$$ M_{\Gamma} = \llrr{C^{+\Gamma} \sep C\subseteq_{fin} \BQ}.$$
Clearly $M_{\Gamma}$ is non-empty,
and $\Gbkv\in M_{\Gamma}$.
Also,
for all $X\in M$ we have $\Gbkv\subseteq X$,
so in other words,
any finitely generated closed set contains all variables with known value.
Then, we have the following disjoint decomposition of $\BQ$ using $X\in M_{\Gamma}$:
$$\BQ = \Gbkv \cup (X\backslash \Gbkv) \cup (\BQ\backslash X).$$
Intuitively,
the values of the variables in $\Gbkv$ must hold fixed among all possible worlds;
the values of the variables in $X\backslash \Gbkv$ must vary relative to those in $\Gbkv$
in a uniform way to respect the functional dependencies among them;
and the values of the variables in $\BQ\backslash X$ must vary
even when all values in $X\backslash \Gbkv$ are fixed,
since they are not determined by $X$.

For example,
suppose $\BQ = \llrr{a, b, c, d}$, $\BG = \mathbb{N}$,
and we want to model $\Gamma$ whose knowledge consists only of: 
\[\Kv(a),\Kf(b, c)\]
and their logical consequences such as $\Kf(c, a)$.
Then, when considering $X = \llrr{a, b, c} = \llrr{b}^{+\Gamma}$,
we have $\Gbkv = \llrr{a}$,
$X\backslash\Gbkv = \llrr{b,c}$,
and $\BQ\backslash X = \llrr{d}$.
Among all possible worlds,
the value of $a$ must be fixed;
$c$ must change as $c\not\in \Gbkv$,
but it should change together with $b$ in case of violating functionality;
and $d$ has to change even when $b$ together with $c$ are fixed to refute $\Kf(b, d)$.
Thus, one instantiation of this could be:

\begin{center}
    \begin{tabular}{|r|c|c|c|}
      \hline
      $\Gbkv$ & $a = 0$ & $a = 0$ & $a = 0$ \\
      \hline 
      $X\backslash\Gbkv$ & $b = 0$ & $b = 1$ & $b = 1$ \\
              & $c = 0$ & $c = 1$ & $c = 1$ \\
      \hline
      $\BQ\backslash X$ & $d = 0$ & $d = 1$ & $d = 2$ \\
      \hline
    \end{tabular}
\end{center}
where the columns are possible assignments.
For every $X\in M_{\Gamma}$ which collects all closed set of variables,
we need such possibilities to take care of all formulas of the form $\lnot\Kf(C, d)$ in $\Gamma$,
because there will be one $X$,
namely $C^{+\Gamma}$,
that separates $C$ and $d$.
Then, the value of $d$ can vary even when those of $C$ are fixed.

The reason we are using only finitely generated closed subsets of $\BQ$ is that,
when $|\BQ|$ is infinite,
the cardinality remains the same.
Formally, define $\powerset_f(\BQ)$ to be the collection of all finite subsets of $\BQ$,
then $|\powerset_f(\BQ)| = |\BQ|$ when $|\BQ| \ge \aleph_0$.
Of course, when $\BQ$ is finite,
$\powerset_f$ coincides with $\powerset$,
the ordinary powerset construction.
Then, by the definition of $M_{\Gamma}$, $|M_{\Gamma}| \le |\powerset_f(\BQ)|$.

Now suppose $|\BG| \ge |\powerset_f(\BQ)\times\llrr{0, 1}|$,
which is the largeness condition for $\BG$ in this case,
then there exists an injection $g: M_{\Gamma}\times\llrr{0,1} \to \BG$.
Using this $g$ we can define a function $V_p$ on $M_{\Gamma}\times\llrr{0, 1}\times \BQ$ as follows:
$$
V_p(\lr{X, i}, d) = \left\{
    \begin{array}{ll}
      g(\emptyset, 0)     & d\in\Gbkv \\
      g(X, 0)             & d\in X\backslash \Gbkv \\
      g(X, i)             & d\in \BQ\backslash X .
    \end{array}
\right. 
$$
Notice how this satisfies the informal requirement, illustrated by the example above,
over the values the variables in different regions should take.
When $d\in\Gbkv$, its value is fixed to $g(\emptyset, 0)$.
When $d\in X\backslash\Gbkv$,
its value depends on $X$ as a whole but nothing else,
so all variables in $X\backslash\Gbkv$
change uniformly
from what they are assigned by $g(\emptyset, \cdot)$.
When $d\in \BQ\backslash X$,
its value further depends on $i$,
so will change even when the values of the variables in $X$ are fixed. 

Formally, this definition allows us to show:
\begin{proposition}\label{kvkf} For all $C\subseteq_{fin} \BQ, d\in \BQ$:
    \begin{enumerate}
        \item If $\Kv(d)\in\Gamma$ then 
        $$\exists x \in G,  \forall \lr{X, i}\in M\times\llrr{0, 1},
        V_p(\lr{X, i}, d) = x ;$$
        
        \item If $\Kv(d)\not\in\Gamma$ then 
        $$\exists \lr{X, i}, \lr{X', i'}\in M\times\llrr{0, 1},
        V_p(\lr{X, i}, d) \not= V_p(\lr{X', i'}, d) ;$$
        
        \item If $\Kf(C, d)\in\Gamma$ then 
        \begin{multline*}
            \forall \lr{X, i}, \lr{X', i'}\in M\times\llrr{0, 1},  \\
            V_p(\lr{X, i}, C) = V_p(\lr{X', i'}, C)
            \Rightarrow
            V_p(\lr{X, i}, d) = V_p(\lr{X', i'}, d);
        \end{multline*}

        \item
        If $\Kf(C, d)\not\in\Gamma$ then
        \begin{multline*}
            \exists \lr{X, i}, \lr{X', i'}\in M\times\llrr{0, 1},  \\
            V_p(\lr{X, i}, C) = V_p(\lr{X', i'}, C)
            \textrm{\ and\ }
            V_p(\lr{X, i}, d) \not= V_p(\lr{X', i'}, d).
        \end{multline*}
    \end{enumerate}
\end{proposition}
\begin{proof}
    For the first part,
    the witness is $x = g(\emptyset, 0)$ and can be verified easily.
    For the second part, as we observed before,
    $\Gbkv = \emptyset^{+\Gamma} \in M_{\Gamma}$.
    Then, 
    if $d\not\in\Gbkv$,
    on $\lr{\Gbkv, 0}$ and $\lr{\Gbkv, 1}$ 
    our valuation function $V_p$ gives different values by the injectivity of $g$.

    For the third part, two cases are possible.
    If $C\subseteq\Gbkv$,
    then $d\in\Gbkv$ by $\ttt{VF}$.
    Then $V_p$ assigns $g(\emptyset, 0)$ to $d$ on all $\lr{X, i}$,
    making the consequent of the implication to be proven true throughout.

    Now suppose $C\not\subseteq\Gbkv$ and
    take $c\in C\backslash(\Gbkv)$ and
    $\lr{X, i}, \lr{X', i'}\in M\times\llrr{0, 1}$ such that
    $V_p(\lr{X, i}, C) = V_p(\lr{X', i'}, C)$.
    We first show $X = X'$ by focusing on this $c\not\in\Gbkv$.
    Since $c\not\in\Gbkv$,
    by the definition of $V_p$,
    there exists $j, k\in \llrr{0, 1}$ such that
    $$V_p(\lr{X, i}, c) = g(X, j),  V_p(\lr{X', i'}, c) = g(X', k). $$
    By the injectivity of $g$,
    they are equal only if at least $X = X'$.
    Based on this,
    if $i = i'$ then $\lr{X, i} = \lr{X', i'}$ and trivially $d$ receives the same value from $V_p$.

    If $i \not= i'$,
    recall that we assumed $V_p(\lr{X, i}, C) = V_p(\lr{X, i'}, C)$.
    For all $c\in C$,
    it follows that $c\in X$ as otherwise
    the values $V_p$ gives to $c$ differ on $i$ and $i'$.
    Hence $C\subseteq X$ and by assumption $X\in M_{\Gamma}$,
    which means $X$ is closed.
    Thus, as $\Kf(C, d)\in\Gamma$, $d\in X$ as well.
    By definition, $$V_p(\lr{X, i}, d) = g(X,  0) = g(X', 0) = V_p(\lr{X', i'}, d). $$

    For the last part,
    we assume that $\Kf(C, d)\not\in\Gamma$.
    Then $d\not\in C^{+\Gamma}$.
    By the injectivity of $g$ and the fact that $C\subseteq C^{+\Gamma}$,
    $$ V_p(\lr{C^{+\Gamma}, 0}, d)\not= V_p(\lr{C^{+\Gamma}, 1}, d),$$
    whereas
    $$V_p(\lr{C^{+\Gamma}, 0}, C) = V_p(\lr{C^{+\Gamma}, 1}, C). $$\qed
\end{proof}
The above proposition handles the knowledge and ignorance about values and functional dependencies.
Now we need to combine it with a traditional completeness proof for epistemic S5 logic.
Denote
\[L :=
    \llrr{\Delta \sep
      \Delta \textrm{\ is\ maximal\ consistent\ in\ }
      \axiomLKVF+\mathtt{EXT} \textrm{\ and\ }\Gbk\subseteq \Delta}
    .
\]
Here $L$ is non-empty since by axiom $\mathtt{T}$,
$\Gbk \subseteq \Gamma$ so at least $\Gamma\in L$.
Then we define a model on possible worlds $W = L\times M_{\Gamma}\times \llrr{0, 1}$:
$\modelM = \lr{W,  U, V}$ where for every $\lr{\Delta, C, i}\in W$:
$$
\begin{array}{lll}
  
  U(\lr{\Delta, C, i}, p) &=& [p \in \Delta] \\
  V(\lr{\Delta, C, i}, d) &=& V_p(\lr{C, i}, d)
\end{array}
$$
where $[p \in \Delta]$ is the indicator function of the statement $p\in \Delta$,
which evaluates to 1 if the statement is true and 0 otherwise.
Here each possible world has three components:
a maximally consistent set which contains all formulas true at the world (truth lemma),
a closed set of variables $C$
which is responsible for
instantiating the ignorance of the values of variables in $C$
under the functional dependency constraint,
and a number $0$ or $1$ which is responsible for
instantiating the ignorance of
the functionality property between variables in $C$ and variables outside $C$.

Now the goal is to show a truth lemma, i.e.,
for all $\lr{\Delta, C, i}\in W, \phi\in \Delta \LRA \modelM, \lr{\Delta, C, i}\vDash \phi$.
To this end,
we first need the following simple observation.             
\begin{proposition}\label{45}
    For all $\Delta\in L$,
    \begin{itemize}
        \item $\Kv(d)\in \Delta \LRA \Kv(d)\in \Gamma$
        \item $\Kf(C, d)\in \Delta \LRA \Kf(C, d)\in \Gamma$
        \item $K\phi \in \Delta \LRA K\phi \in \Gamma.$
    \end{itemize}
\end{proposition}
\begin{proof}
    Simply use the axioms $\mathtt{4, 5}$. For example, the third property follows from
    $$
    \begin{array}{llll}
      K\phi\in\Gamma & \Rightarrow &KK\phi\in\Gamma   &[\textrm{axiom\ }\mathtt{4}]\\
                     & \Rightarrow &K\phi \in \Gbk    &[\textrm{definiton\ of\ }\Gbk]\\
                     & \Rightarrow &K\phi \in \Delta       &[\Gbk \subseteq \Delta] \\
                     &&& \\    
      K\phi\not\in\Gamma
                     & \Rightarrow &\lnot K\phi\in\Gamma& [\Gamma\textrm{\ is\ maximally\ consistent}]\\
                     & \Rightarrow &K\lnot K\phi\in\Gamma  &[\textrm{axiom }\mathtt{5}]\\
                     & \Rightarrow &\lnot K\phi \in \Gbk   &[\Gbk\textrm{\ definition}]\\
                     & \Rightarrow &\lnot K\phi \in \Delta      &[\Gbk \subseteq \Delta] \\
                     & \Rightarrow &K\phi \not\in \Delta  &[\Delta\textrm{\ is\ maximally\ consistent}].
    \end{array}
    $$
\end{proof}

\begin{proposition}\label{notK}
    If $K\phi\not\in\Gamma$, then there exists $\Delta\in L$ such that  $\lnot\phi\in \Delta$.
\end{proposition}
\begin{proof}
    A standard exercise using necessitation and axiom $\mathtt{K}$.
\end{proof}

Now we can prove the truth lemma:         
\begin{lemma}\label{Truth1}
    For all $\lr{\Delta, C, i}\in W,
    \phi\in \logicLKVF$,
    $\phi\in \Delta \LRA \modelM, \lr{\Delta, C, i}\vDash \phi$.
\end{lemma}
\begin{proof}
    By induction on $\phi$, with the following possibilities:
    \begin{itemize}
        \item $\phi$ is a propositional letter or a boolean combination. This is standard. 
        \item $\phi = \Kv(d)$.
        Since $\Delta\in L$,
        by Proposition \ref{45},
        $\Kv(d)\in \Delta\LRA \Kv(d)\in\Gamma$.
        By Proposition \ref{kvkf},
        if $\Kv(d)\in\Gamma$ then
        $$
        V(\lr{\Theta, D, j}, d) =
        V_p(\lr{D, j}, d) =
        V_p(\lr{D', j'}, d) =
        V(\lr{\Theta', D', j'}, d)
        $$
        for all $\lr{\Theta, D, j}, \lr{\Theta', D', j'}\in W$.
        If $\Kv(d)\not\in\Gamma$,
        by Proposition \ref{kvkf} again,
        there exists $\lr{D, j}, \lr{D', j'}\in M\times\llrr{0, 1}$ such that
        $$
        V(\lr{\Gamma, D, j}, d) = V_p(\lr{D, j}, d) \not=
        V_p(\lr{D', j'}, d) = V(\lr{\Gamma, D', j'}, d).
        $$
        As such,
        $$
        \Kv(d)\in \Delta
        \LRA
        \Kv(d)\in\Gamma
        \LRA
        \modelM, \lr{\Delta, C, i}\vDash \phi .
        $$
        \item $\phi = \Kf(D, d)$.
        Similar to the last one.
        By Proposition \ref{45},
        $\Kf(D, d)\in \Delta\LRA \Kf(D, d)\in\Gamma$.
        By Proposition \ref{kvkf},
        $\Kf(D, d)\in\Gamma\LRA \modelM,
        \lr{\Delta, C, i}\vDash\Kf(D, d)$.
        \item $\phi = K\psi$.
        By Proposition \ref{45},
        $K\psi\in \Delta \LRA K\psi\in\Gamma$.
        If $K\psi\in\Gamma$, then $\psi\in\Gbk$,
        so for all $\lr{\Theta, D, j}\in W$,
        as $\Theta\in L$, $\psi\in \Theta$.
        By the induction hypothesis,
        $\modelM, \lr{\Theta, D, j}\vDash\psi$.
        Thus, $\modelM, \lr{\Delta, C, i}\vDash K\psi$.

        On the other hand,
        if $K\psi\not\in\Gamma$,
        by Proposition \ref{notK},
        there exists $\Theta\in M$ such that $\lnot\phi\in \Theta$.
        By the induction hypothesis,
        $\modelM, \lr{\Theta, \emptyset, 0}\vDash\lnot\psi$.
        So $\modelM, \lr{\Delta, C, i}\not\vDash K\psi$.
        To sum up, $K\psi\in\Gamma\LRA \modelM, \lr{\Delta, C, i}\vDash K\psi$.
    \end{itemize}
\end{proof}
From this proposition,
we know that for all $\phi\in\Gamma$,
$\modelM, \lr{\Gamma, \emptyset, 0}\vDash\phi$.
As the consistent set $A$ we chose at the very beginning
is contained in $\Gamma$,
$\modelM, \lr{\Gamma, \emptyset, 0}\vDash A$,
which brings us:    
\begin{theorem}
    Given $|\BG| \ge |\powerset_f(\BQ)\times\llrr{0,1}|$ and
    $\BF = \bigcup\llrr{\BG^{\BG^i}\ |\ i\in \mathbf{N}}$,
    $\axiomLKVF+\mathtt{EXT}$ axiomatizes $\logicLKVF$.
\end{theorem}

\section{Minimal Function Domain}

In Proposition \ref{SoundCondition}
we proved the soundness condition for $\axiomLKVF$.
Notice that the minimal function domain that satisfies this soundness condition is 
$$\BF = \llrr{id_{i,j}\sep i,j\in\mathbb{N}, 0 < i \le j}.$$
In this section,
we consider the axiomatization of the validities of $\logicLKVF$ with this $\BF$.
Here, two axioms besides our base system $\axiomLKVF$ are valid:
$$
\begin{array}{ll}
  \mathtt{CHOO} & \Kf(C,d)\to\bigvee_{c\in C}\Kf(c,d),\\
  \mathtt{EQU} & \Kf(c,d)\to \Kf(d,c).
\end{array}
$$
The validity of the first axiom is justified by:
$$d = id_{i,j}(c_1,c_2,\ldots,c_j) = c_i = id_{1,1}(c_i),$$
and notice that when $C = \emptyset$,
it degenerates to $\Kf(\emptyset, d) \to \bot$
or equivalently $\lnot\Kf(\emptyset, d)$,
which is true because no zero-ary function exists in $\BF$.
This also means that $\mathtt{EXT}$ is unsound in this case,
because even if $\Kv(d)$ is true,
$\Kf(\emptyset, d)$ is false regardless.
So $\Kv(d)\to \Kf(C, d)$ is in general false.

The validity of the second axiom follows from 
$$d = id_{1,1}(c) = c \quad \Rightarrow \quad c = d = id_{1,1}(d).$$
Thus, $\axiomLKVF + \mathtt{CHOO} + \mathtt{EQU}$ is sound.
Given these two axioms and the fact that
$\BF$ consists only of projection functions,
$\Kf(c, d)$ is actually talking about the equality of $c, d$ over all possible worlds,
even though the value might not be known.
This motivates the construction of the equivalence relation by $\Kf(c,d)$ used below. 

Now we turn to the proof of the completeness of $\axiomLKVF + \mathtt{CHOO} + \mathtt{EQU}$.
Again, given a consistent set $A$,
our plan is that we first extend it to a maximal consistent set $\Gamma$,
then deal with its \textit{de re} knowledge and propositional knowledge separately,
and finally take their Cartesian product to obtain a model of $\Gamma$.

First,
we partition $\BQ$ into equivalence classes
with equivalence relation $\sim$ defined by
$$c\sim d \LRA \Kf(c,d)\in\Gamma. $$
Its reflexivity, symmetry and transitivity follow from
the axioms $\mathtt{PROJ}$, $\mathtt{EQU}$, and $\mathtt{TRAN}$.
Indeed, if we use the $C^{+\Gamma}$ and $M_{\Gamma}$ construction,
$M_{\Gamma}$ will contain precisely those partitions and their unions.
Every maximally consistent set, or a ``world'',
naturally gives rise to such an equivalence relation on $\BQ$.

For every $c\in \BQ$,
define $[c] = \llrr{d \sep c\sim d}$,
and for every $C\subseteq \BQ$,
define $[C] = \lr{[c]\sep c\in C}$,
the collection of the equivalence classes
which contain at least one of its elements.
In particular,
$[\Gbkv] = \lr{[c]\sep \Kv(c)\in \Gamma}$.

Now,
if $|\BG| \ge |\BQ| \ge |[\BQ]|$,
then there will be two injections from $[\BQ]$ to $\BG$, $u$ and $v$,
such that 
$$u([c]) = v([c]) \LRA [c]\in [\Gbkv].$$
For example,
we can let $u$ be any injection and then
make a rotation over the function values of $u$ on $[\BQ]\backslash[\Gbkv]$
to obtain $v$ in case of $\BQ$ being finite,
or let $v(d)$ be the successor of $u(d)$ for $d\in [\BQ]\backslash[\Gbkv]$
in case of $\BQ$ being infinite (assuming it can be well ordered).
We do not need to seek more valuations of variables
to prove the truth lemma in this case
or to instantiate the ignorances of the knowledge about values in $\Gamma$.
Any one of them is capable of refuting $\Kf(C, d)\not\in\Gamma$
and together they instantiate $\Kv(d)\not\in\Gamma$. 

Defining $V_p$ as a function from $\llrr{u,v}\times\BQ$ to $\BG$ by $V_p(t,d) = t([d])$,
the following proposition is true:
\begin{proposition}\label{kvkfMin}
    For any $d\in \BQ, C\subseteq_{fin} \BQ$:
    \begin{enumerate}
        \item if $\Kv(d) \in \Gamma$,
        $\exists x\in G, \forall t\in \llrr{u,v}, V_p(t, d) = x$

        \item if $\Kv(d)\not\in\Gamma$,
        $\exists t, t'\in\llrr{u,v}, V_p(t,d)\not= V_p(t', d)$

        \item if $\Kf(C, d)\in\Gamma$,
        $\exists f\in \BF,\forall t\in\llrr{u,v},f(V_p(t,C))=V_p(t, d)$

        \item if $\Kf(C, d) \not\in\Gamma$,
        $\forall f\in \BF,  \exists t\in \llrr{u,v},  f(V_p(t, C)) \not= V_p(t, d)$.
    \end{enumerate}
\end{proposition}
\begin{proof}
    The first two parts are immediate from the definition of $u,v$:
    $\Kv(d)\in\Gamma \LRA [d]\in[\Gbkv] \LRA u([d]) = v([d]) \LRA V_p(u,d) = V_p(v,d)$.

    For the third property,
    suppose $\Kf(C,d)\in\Gamma$ and
    enumerate $C$ by $c_1,\ldots,c_j$.
    By axiom $\mathtt{CHOO}$ and the maximality of $\Gamma$,
    there exists $i$ such that $\Kf(c_i,d)\in\Gamma$ and thus $[d] = [c_i]$.
    Now, for every $t\in\llrr{u,v},
    V_p(t, C) = \lr{[c_1], [c_2],\ldots,[c_j]}$,
    so $[d] = id_{i,j}(V_p(t, C))$ and we see that
    the functional relation between $C,d$ is $id_{i,j}$.

    For the last one,
    suppose $\Kf(C,d)\not\in\Gamma$.
    It follows that $[d]\not\in[C]$ because otherwise,
    $[d]\in[C]$ and there exists $c\in C, [d] = [c]$,
    hence $\Kf(c,d)\in\Gamma$.
    By axiom $\mathtt{PROJ}$,
    $\Kf(C,c)\in\Gamma$,
    and then by axiom $\mathtt{TRAN}$,
    $\Kf(C,d)\in\Gamma$,
    which contradicts the assumption.
    Again enumerate $C = \lr{c_1,\ldots,c_j}$.
    Since $u$ is injective and $[d]\not\in [C]$,
    for all $c_i\in C$, $u([d]) \not= u([c_i])$.
    Thus, for every $j$-ary function $id_{i,j} \in \BF$,
    $id_{i,j}(V_p(u,C)) = u([c_i]) \not= u([d])$.
    Actually we can use $v$ here as well.
    The reason we need both of them is that
    we need to instantiate $\lnot\Kv(d)$ for $d\not\in\Gbkv$.
\end{proof}

To build a model for $\Gamma$, define     
$$
\begin{array}{rll}
  L & = & \llrr{\Delta \sep \Delta \textrm{\ is\ a\ maximal\ consistent\ set}\
          ,\Gbk\subseteq \Delta} \\
  W & = & L\times \llrr{u,v} \\
  U(\lr{X, t}, p) & = & [p \in X] \\
  V(\lr{X, t}, d) & = & V_p(t, d)\\
  \modelM & = & \lr{W, U, V}.
\end{array}$$
Then we have the following truth lemma:
\begin{lemma}
    For all $\lr{\Gamma, t}\in W$,
    $\lr{\Gamma, t} \vDash \phi$ if and only if $\phi\in\Gamma$.
\end{lemma}
\begin{proof}
    The proof is similar to that of Lemma \ref{Truth1}.
    The difference is that
    we need to use Proposition \ref{kvkfMin} instead of Proposition \ref{kvkf}.
\end{proof}

The completeness of $\logicLKVF + \mathtt{CHOO} + \mathtt{EQU}$ follows, so we conclude:
\begin{theorem}
    Given
    $|\BG| \ge |\BQ|,
    \BF = \lr{id_{i,j}\sep i,j\in\mathbb{N}, 0 < i \le j}$,
    $\axiomLKVF + $ axiomatizes $\logicLKVF$.
\end{theorem}

\section{Intermediate Function Domain}

In the previous two sections,
we considered the minimal and the maximal function domains
subject to our soundness condition.
As we can see,
in both cases
the axiomatizations require some axioms besides the base system $\axiomLKVF$.
And those axioms are not very intuitive
if we intend to interpret $\Kf$ as ``knowing a/the functional dependency''.
In this section,
we show that we can construct a function domain such that if $\BF$ is set to it,
$\axiomLKVF$ will be complete and no extra axiom is needed.
The construction is somewhat artificial
but in the next section,
we can view this as just one step of a completeness proof at a higher level.

The main difficulty here is to refute the axiom scheme $\mathtt{EXT}$
used in the axiomatization of the full function domain case.
$\ttt{EXT}$ is validated in that case because whenever the value of a variable is known,
a constant function can be used to explain the functional dependency between
it and any other variables in all epistemic possibilities.
Thus, to refute this scheme as an axiom,
we must make sure that the function domain encodes information more than just functionality
so that we can refute $\Kf(c, d)$ even when functionality holds,
such as when $\Kv(d)$ is true.
The function domain to be constructed below will enable a suitably constructed model
to refute $\Kf(C,d)$ without ever looking into the functionality condition. 

To do this, we go to higher dimensions by assuming $\BG = 2^{\powerset_f(\BQ)}$,
interpreted as functions from the finite subsets of $\BQ$ to $\llrr{0, 1}$
or as a rather long sequence indexed by $\powerset_f(\BQ)$
where at each index (dimension) $C$ we can choose from $\llrr{0,1}$.
This is actually only a size requirement,
since so long as $|\BG| \ge |2^{\powerset_f(\BQ)}|$,
we can always embed $2^{\powerset_f(\BQ)}$ into $\BG $ by an injection.
For any $x\in\BG$ and $C\subseteq_{fin} \BQ$, we use $x[C]$ to retrieve the image of $C$ under $x$,
which will be $0$ or $1$.
Now we construct the intermediate $\BF$:

\begin{definition}
    Let $\BF$ be the collection of the functions $f$ satisfying the following constraints:
    where $y$ is $f(x_1, \ldots x_n)$, for all $C\subseteq_{fin} \BQ $,
    $$x_1[C] = x_2[C] = \ldots = x_n[C] = 0 \quad \Rightarrow \quad y[C] = 0.$$

    \noindent Alternatively, where
    $$\Lmax_n =
    \llrr{f\in\BG^{\BG^n} \sep
      \forall C\subseteq_{fin} \BQ,
      f(x_1, \ldots x_n)[C] \le \max(x_1[C], \ldots x_n[C])},$$
    with $\max() = 0$, define $\BF = \bigcup_{i\in\mathbb{N}} \Lmax_i$.
\end{definition}

Notice that the requirement is specified for all dimensions individually,
and they do not interfere with each other.
This allows us to do constructions and proofs for each dimension separately. 

Now we can check that this $\BF$ satisfies the soundness condition.
Projection functions are all included in $\BF $ because they all satisfy the above constraint:
for any $C\subseteq_{fin}\BQ $,
either $x_i[C] = y[C] = 1$, where the antecedent and the consequent are both false,
or $x_i[C] = y[C] = 0$, where they are both true.
For compositionality, let $h = f(g_1, \ldots g_n)$.
If all inputs to $h$ are $0$ at any dimension $C$,
then since $g_1, \ldots g_n \in \BF$,
they evaluate to $0$ at dimension $C$.
Then all inputs to $f$ are $0$ at this dimension $C$.
So as $f\in\BF $, it evaluates to $0$ as well.
Thus, $h$ is in $\BF$.

To prove the completeness of $\axiomLKVF$ with respect to $\logicLKVF$ with this new function domain $\BF$,
again the satisfiability of any maximal consistent set $\Gamma $ is required,
and the crucial step is still the construction of a set of valuations
such that the formulas of the form $\Kv(d), \lnot\Kv(d), \Kf(C, d)$, and $\lnot\Kf(C, d)$ in $\Gamma $ are satisfied.
Indeed, for this purpose, we only need two valuations,
a situation similar to that in the case of the minimal function domain.
This is because when $\lnot\Kf(C, d)\in\Gamma$,
we are refuting $\Kf(C, d)$ not by a failure of functionality
but by a failure of conformation to $\BF$.
Breaking functionality requires at least two possible value assignments,
but if $\BF$ says no, a single possibility is too many.
Recall the $C^{+\Gamma}$ we used in the previous two cases,
which is defined as $\llrr{d\in\BQ \sep \Kf(C, d) \in \Gamma}$.
Now we need to define a slightly different $M_{\Gamma}$:
$$ \llrr{C^{+\Gamma} \sep C\subseteq_{fin} \BQ} \cup \llrr{\Gbkv}.$$
This is the collection of all finitely generated closed sets plus $\Gbkv$.
We need this extra union since axiom $\ttt{EXT}$ is not available now,
which means $\Gbkv$ is not automatically contained in any $C^{+\Gamma}$,
and it is quite possible that $\Gbkv$ is not finitely generated.
But still, $M_{\Gamma}$ has a cardinality no larger than $\powerset_f(\BQ)$,
since if $\BQ$ is finite, $\powerset_f(\BQ)$ contains all subsets of $\BQ$,
and if infinite,
$\powerset_f(\BQ)$ is also infinite
and adding one more element into it does not increase its cardinality.
Thus, there is still a surjection $g$ from $\powerset_f(\BQ)$ to $M_{\Gamma}$.
We can think of this $g$ as a pseudo $(\cdot)^{+\Gamma}$ function,
and it does not matter which surjection we use for $g$.
Now we can specify the two valuations we need:
\begin{definition}
    Let $g$ be any surjection from $\powerset_f(\BQ)$ to $M_{\Gamma}$.
    Define $V_0, V_1:\powerset_f(\BQ) \to\BG$ such that for all $d\in\BQ, C\subseteq_{fin}\BQ$,
    \[V_0(d)[C] = \left\{
            \begin{array}{ll}
              0 & \textrm{if\ } d\in g(C) \\
              1 & \textrm{if\ } d\not\in g(C),
            \end{array}
        \right.\]
    
    \[V_1(d)[C] = \left\{
            \begin{array}{ll}
              V_0(d)[C] & \textrm{if\ } g(C)\not=\Gbkv \\
              0 & \textrm{if\ } g(C) = \Gbkv.
            \end{array}
        \right. \]
    
\end{definition}
The use of $V_0$ is to refute $\Kf(C, d)$ if $\lnot\Kf(C, d)\in\Gamma$,
and the use of $V_1$ is to refute $\Kv $ if $\lnot\Kv(d)\in \Gamma$. 
Now we prove this in detail:

\begin{proposition}\label{kfInt}
    If $\Kf(C,d)\in\Gamma$,
    then there exists $f\in\BF $ such that
    for $i\in\llrr{0,1}$, $f(V_i(C)) = v_i(d)$.
    If $\lnot\Kf(C,d)\in\Gamma$,
    then for all $f\in\BF$, $f(V_0(C)) \not= V_0(d)$.
\end{proposition}

\begin{proof}
    To prove the first claim,
    assume $\Kf(C,d)\in\Gamma $ with $C$ enumerated by $c_1, \ldots c_n$.
    We will construct a function $f \in \BF $ that works in both $V_0$ and $V_1$:
    for all $D\subseteq_{fin} \BQ$,
    $V_0(d)[D] = f(V_0(C))[D]$ and $V_1(d)[D] = f(V_1(C))[D]$.
    Obviously this construction should be done dimension by dimension.
    For any $D\subseteq_{fin} \BQ$, the possibilities are:
    \begin{itemize}
        \item $d \in g(D)$.
        Thus, by definition, $V_0(d)[D] = 0$.
        $V_1(d)[D] = 0$ as well since the only change happens when $D = \Gbkv$,
        and even in that case, only $1$ turns to $0$ and not vice versa.
        So we can define $f(x_1, \ldots x_n)[D] = 0$.
        Then $V_0(d)[D] = f(V_0(C))[D]$ and $V_1(d)[D] = f(V_1(C))[D]$,
        regardless of what $V_0(C)$ and $V_1(C)$ are. 
        \item $d\not\in g(D)$.
        Since $g(D)$ is closed and $\Kf(C, d)\in \Gamma$,
        $C\not\subseteq g(D)$.
        Find $c_p\not\in g(D)$.
        Define $f(x_1, \ldots x_n)[D] = x_p[D]$.
        This definition satisfies the requirement of $\BF$.
        And it works for $V_0$ because $v_0(d)[D] = V_0(c_p)[D] = 1$ (both $d, c_p$ are outside $g(D)$).
        It also works for $V_1$ because their values change to $0$ together if $g(D) = \Gbkv$.
    \end{itemize}

    To prove the second claim,
    recall that $C^{+\Gamma} = \llrr{d \sep \Kf(C,d)\in\Gamma}$
    is closed under $\Kf$ in $\Gamma$ and
    contains $C$ by axioms $\ttt{TRAN}$ and $\ttt{PROJ}$.
    Now since $\Kf(C, d)\not\in\Gamma$, $d\not\in C^{+\Gamma}$.
    As $g$ is a surjection from $\powerset_f(\BQ)$ to $M_{\Gamma}$,
    there exists $D\subseteq_{fin} \BQ$ such that $g(D) = C^{+\Gamma}$.
    Thus, by the definition of $V_0$,
    $V_0(d)[D] = 1$, while for all $c\in C \subseteq C^{+\Gamma} = g(D)$, $V_0(c)[D] = 0$.
    Hence $V_0(d)[D] > \max(V_0(C)[D])$,
    which makes it impossible to find a function $f\in\BF$ such that $f(V_0(C)) = V_0(d)$.
\end{proof}

\begin{proposition}\label{kvInt}
    If $\Kv(d) \in\Gamma$, then $V_0(d) = V_1(d)$.
    If $\Kv(d)\not\in\Gamma$, then $V_0(d) \not= V_1(d)$.
\end{proposition}
\begin{proof}
    If $\Kv(d) \in \Gamma$, then $d\in\Gbkv$.
    Now for any $C\subseteq_{fin} \BQ $,
    if $g(C) \not= \Gbkv$,
    then $V_1(d)[C] = V_0(d)[C]$ by definition.
    If $g(C) = \Gbkv$,
    $V_1(d)[C] = 0$, but $V_0(d)[C] = 0$ as well since $d\in\Gbkv$.
    Thus, $V_0(d) = V_1(d)$.

    If $\Kv(d) \not\in\Gamma$, $d\not\in\Gbkv$.
    Since we explicitly added $\Gbkv$ to  $\Gamma$, $\Gbkv\in M_{\Gamma}$,
    and we can find a $C\subseteq_{fin} \BQ$ such that $g(C) = \Gbkv$.
    Then, using the definition of $V_0$ and $V_1$,
    we know $V_0(d)[C] = 1$ but $V_1(d)[C] = 0$,
    because $g(C) = \Gbkv$ and we assumed $d\not\in\Gbkv$.
    Thus, $V_1(d) \not= V_0(d)$.
\end{proof}

Based on the previous two propositions, 
we can build a model for $\Gamma$ by defining
$$
\begin{array}{rll}
  L & = & \llrr{\Delta\sep \Delta\textrm{\ is\ a\ maximal\ consistent\ set},\Gbk\subseteq \Delta} \\
  W & = & L\times \llrr{0,1} \\
  U(\lr{X, t}, p) & = & [p \in X] \\
  V(\lr{X, t}, d) & = & V_t(d)\\
  \modelM & = & \lr{W, U, V}.
\end{array}
$$
With a proof which is essentially the same as
the proof of the truth lemma Lemma \ref{Truth1} in the full function domain case,
using Propositions \ref{kfInt} and \ref{kvInt} instead of Proposition \ref{kvkf},
we have:
\begin{lemma}
    For all $\lr{\Gamma, t}\in W$, $\modelM, \lr{\Gamma, t}\vDash \phi$
    if and only if $\phi\in\Gamma$.
\end{lemma}
$\modelM, \lr{\Gamma, 0} \vDash \Gamma$ follows from this truth lemma.
This finishes the completeness proof of the intermediate case, so we have:
\begin{theorem}
    Given $|\BG| \ge |2^{\powerset_f(\BQ)}|$,
    $\BF = \bigcup_{i \in\mathbb{N}} \Lmax_i$, $\axiomLKVF$ axiomatizes \logicLKVF.
\end{theorem}

\section{Unifying Logic}
In all the previous settings,
our logic $\logicLKVF $ takes a function domain $\BF$ as a parameter.
This function domain is meant to be the set of
\textit{a priori} possible functions for functional dependencies over variables.
But if this set of \textit{a priori} possibilities is relative to the agents in discussion,
then this set of functions should be variable over models
instead of being part of the logic and fixed for all models.
After all, an agent might hold different prior knowledge in different worlds.
\begin{table}
    \centering
    \caption{Choice of the function domain in $\logicLKVF$ and corresponding axiomatization}
    
    \begin{tabu}{rlll}
        \toprule
         & Full & Minimal & Intermediate \\
        \midrule
        $\BF = $ & $\bigcup_{i\in\mathbb{N}}\BG^{\BG^i}$
        & $\llrr{id_{i,j} \sep i,j\in\mathbb{N}, 0 < i \le j}$
        & $\bigcup_{i\in\mathbb{N}}\Lmax_i$ \\
        \midrule
        $|\BG| \ge $
        &$|\powerset_f(\BQ)\times\llrr{0,1}|$
        &$|\BQ|$
        &$|2^{\powerset_f(\BQ)}|$ \\
        
        \midrule
        Axiomatization
        & $\axiomLKVF + \ttt{EXT}$
        & $\axiomLKVF + \ttt{CHOO} + \ttt{EQU}$
        & $\axiomLKVF$ \\
        \bottomrule
    \end{tabu}
    
\end{table}
Also, the function domain constructed in the intermediate case is,
while not nonsensical for its interesting $\le\max$ structure,
still somewhat artificial for its large dimension.
If this function domain is part of the model,
it is at the choice of the agent under discussion.

Indeed, if we put the function domain inside the definition of a model by setting
$$\modelM = \lr{\BF, W, U, V},$$ where
$\BF:\BG \to \BG$ satisfies the soundness condition that it contains all projection functions and
is closed under function composition,
$W$ is a set of possible worlds,
$U$ is an assignment function for propositional letters,
and $V$ is an assignment function for variables,
and we leave the semantics untouched,
then the soundness and completeness of $\axiomLKVF $ follow immediately from the results presented so far.
Using $\logicLKVF^*$ to denote the logic induced by the definition of the models above, we have:
\begin{theorem}
    $\axiomLKVF $ is sound and complete with respect to $\logicLKVF^*$ when $|\BG| \ge |2^{\powerset_f(\BQ)}|$. 
\end{theorem}
\begin{proof}
    Because for every model of $\logicLKVF^*$,
    its function domain satisfies the soundness condition Proposition \ref{SoundCondition},
    $\axiomLKVF$ is sound in all the models of $\logicLKVF^*$.
    This shows the soundness.

    For any set $\Gamma$ maximally consistent with respect to $\axiomLKVF$,
    take the $\BF$ and the model $\modelM$ constructed in the intermediate function domain case.
    Then $\lr{\BF, \modelM} \vDash \Gamma$ and $\lr{\BF, \modelM}$ is a model of $\logicLKVF^*$.
    Thus, every maximal consistent set is satisfiable. \phantom{allallallal}
\end{proof}
The proof above is a direct adaptation of the completeness result in the intermediate function domain case.
In that case, we built a function domain that works for all maximal consistent sets in the sense that
for all maximal consistent sets $\Gamma$,
this same function domain can be used to refute $\Kf(C, d)\not\in\Gamma$ when functionality cannot be used.
This is actually the reason why the cardinality requirement for $\BG$ is very high there.
However, in the current setting where function domains are part of the models,
the only thing needed is
a method to build a function domain for each maximal consistent set $\Gamma$
so that the functional dependency relation between $C, d$ is rejected if $\lnot \Kf(C, d) \in \Gamma$.
The difference will be made more clear in the following multiagent case.

\subsection{Multiagent logic with variable function domain}
Given an index set $\BA$ of agents, to accommodate multiple agents, the language is now expanded to
$$\phi ::= \top \sep p \sep \Kv_i(d) \sep \Kf_i(C, d) \sep \lnot\phi \sep (\phi\land\phi) \sep K_i\phi ,$$
with $p \in \BP, i\in\BA, d\in\BG$, and $ C\subseteq_{fin}\BG$.
The only difference from the single agent language defined in Definition \ref{syntax} is that
now we have for each agent $i$ a separate $\Kv_i$, $\Kf_i$, and $K_i$. 

For semantics, a model is now defined as:
$$\modelM = \lr{W, \lr{\sim_i}_{i\in\BA}, U, V, \lr{\BF_i}_{i\in\BA}}$$
where $\BF_i$ is intended to assign a collection of functional relationships that
agent $i$ deems possible \textit{a priori} to all possible worlds in $W$.
Thus, for all $w\in W, i\in \BA$,
$\BF_i(w)$ is required to include all projection functions and to be closed under function composition.
$\sim_i$ is the epistemic accessibility relation of agent $i$ and
is required to be an equivalence relation on $W$,
the set of possible worlds (complete epistemic scenarios).
Now since $\BF_i$ is supposed to be ``prior knowledge'',
it is also required that if $w \sim_i w'$, then $\BF_i(w) = \BF_i(w')$.
However, we are not assuming that the prior knowledge of any agent is public to other agents,
so it is quite possible that $\BF_j(w) \not= \BF_j(w')$ if $j \not= i$,
even when $w \sim_i w'$.
In a nutshell, $\BF_i$s are not common knowledge.

The semantic clauses are defined similarly with agent indices for knowledge sentences:
$$
\begin{array}{rcl}
  \modelM, w\vDash \Kv_i(d) & \Leftrightarrow  & \exists x\in \BG, \forall w' \sim_i w, V(w', d) = x \\
  \modelM, w \vDash \Kf(C, d) & \LRA & \exists f\in \BF_i(w),  \forall w' \sim_i w', V(w', d) = f[V(w', C)] \\
  \modelM, w \vDash K\phi &\LRA & \forall w' \sim_i w'\Rightarrow \modelM, w'\vDash \phi.
\end{array}
$$
Let $\logicLKVF^*_m$ name this multiagent logic.
Also, let $\axiomLKVF_m$ denote the axiom system adapted from $\axiomLKVF$
with indexed version of those axioms involving knowledge operators.
In particular, no interaction between agents is allowed,
as there are no axioms saying that we can derive any knowledge about other agents from any agent.
We will see that this is precisely because
we allow each agent to possess its own
prior knowledge about possible functional dependencies,
not necessarily known to other agents.
Once we assume that $\BF_i$s are common knowledge,
interactions will arise,
and we will discuss this point in the last section. 

\newcommand{\kig}{\ensuremath{K_{i, \Gamma}}}
\newcommand{\kvig}{\ensuremath{\Kv_{i, \Gamma}}}
\newcommand{\kigp}{\ensuremath{K_{i, \Gamma'}}}
\newcommand{\kvigp}{\ensuremath{\Kv_{i, \Gamma'}}}
\newcommand{\kigw}{\ensuremath{K_{i, \Gamma_1}}}
\newcommand{\kvigw}{\ensuremath{\Kv_{i, \Gamma_1}}}
\newcommand{\kigt}{\ensuremath{K_{i, \Gamma_2}}}
\newcommand{\kvigt}{\ensuremath{\Kv_{i, \Gamma_2}}}

The soundness of $\axiomLKVF_m$ with respect to $\logicLKVF^*_m$ follows from
an indexed version of Proposition \ref{SoundCondition}.
For completeness we need a new construction:

\begin{definition}[Dependency lattice]\label{lattice}
    Given a maximal consistent set $\Gamma$ in $\logicLKVF_m$ and
    an agent index $i$,
    first define the indexed version of the $(\cdot)^{+\Gamma}$ operator, $\cl_i^{\Gamma}$,
    on finite subsets of $\BQ$ as
    $$\cl_i^{\Gamma}(C) = \llrr{d \sep \Kf_i(C, d)\in \Gamma}.$$
    Then, extend this operator to $\powerset(\BQ)$ by
    $\cl_i^{\Gamma}(C) := \bigcup\llrr{\cl_i^{\Gamma}(C_f) \sep C_f \subseteq_{fin} C}$.
    When the context is clear,
    we may drop the superscript or subscript of $\cl_i^{\Gamma}$.
    Now this is a finitary closure operator as it satisfies,
    through the axioms of $\axiomLKVF_m$,
    $$
    \cl(C) = \cl(\cl(C)),
    C\subseteq \cl(C),
    C\subseteq D\Rightarrow \cl(C) \subseteq \cl(D).
    $$
    When a set $C\subseteq \BQ$ satisfies $C = \cl(C)$, it is called a closed set.
    A classical result is that
    the collection of all closed sets under a closure operator
    forms a lattice $\lr{L, \land, \lor}$ with
    \begin{align*}
      L &= \llrr{C \subseteq \BQ \sep C = \cl_i^{\Gamma}(C)} \\
      C\land D &= C\cap D \\
      C\lor D &= \cl_i^{\Gamma}(C \cup D),
    \end{align*}
    which we name $\lat_i^{\Gamma}$.
    For all $c\in \BQ$,
    let $\cl_i^{\Gamma}(c)$ stands for $\cl_i^{\Gamma}(\{c\})$
    to save a few brackets.  
\end{definition}

Also, given $\Gamma$,
the indexed version of the propositional knowledge and the value knowledge of agent $i$ is denoted by
$$K_{i,\Gamma} = \llrr{\phi \sep K_i\phi \in \Gamma}, \kvig = \llrr{d \sep \Kv_i(d) \in\Gamma}.$$
Then, it is not hard to see that
$\lat_i^{\Gamma}$ is only dependent on $\kig$, i.e.,
if $\kig = \kigp$ then $\lat_i^{\Gamma} = \lat_i^{\Gamma'}$.
This is because
the closure operator $\cl_i^{\Gamma}$ uses only
the formulas of the form $\Kf_i(C,d)$ in $\Gamma$, and if we assume $\kig = \kigp$,
$$\Kf_i(C, d) \in \Gamma \LRA K_i\Kf_i(C, d) \in \Gamma \LRA K_i\Kf_i(C, d) \in \Gamma' \LRA \Kf_i(C, d) \in \Gamma'$$
for all $C\subseteq_{fin}\BQ$ and $d\in\BQ$.

For the completeness proof to go through,
there is again a cardinality requirement for $\BG$:
$|\BG| \ge |\BQ\times \llrr{0,1}|$,
and without loss of generality,
we identify $\BG$ with $\BQ\times \llrr{0, 1}$.
The $\BQ$ part will be used to construct the function domains
and refute $\Kf(C, d)$,
while the $\llrr{0,1}$ part
will be used for refuting $\Kv(d)$.

To use the $\BQ$ part to construct the function domains,
we need to forget the $\llrr{0,1}$ part.
Define function
$\hs_i^{\Gamma}:\BQ\times \llrr{0,1} \to \lat_i^{\Gamma}$,
$\lr{c, n} \mapsto \cl_i^{\Gamma}(c)$ for each $i, \Gamma$.
This map is forgetful about the second coordinate and
turns a variable name into its closure.
Again the superscript and subscript are dropped when no confusion arises.
Now we are able to define a new version of the $\Lmax$ function set:

\begin{definition}
    Given a maximal consistent set $\Gamma$ and
    an agent index $i$,
    we can construct the dependency lattice $\lat$ and the corresponding $\hs$.
    Then define $F_i(\Gamma)$ to be the collection of all functions $f$ on $\BG$ with any arity $n\in\mathbb{N}$ such that:
    $$\hs(f(x_1, x_2, \ldots x_n)) \le \bigvee \llrr{\hs(x_1), \hs(x_2), \ldots \hs(x_n)},$$
    where $\le$ is defined in $\lat$ by
    $\mathfrak{a} \le \mathfrak{b} \LRA \mathfrak{a} \land \mathfrak{b} = \mathfrak{b}$, or equivalently,
    $\mathfrak{a} \subseteq \mathfrak{b}$.
    The empty disjunction is the bottom element of $\lat$: $\cl(\emptyset)$.
\end{definition}

It is straightforward to see that
$F_i(\Gamma)$ is dependent only on $\kig$.
Then we need to verify the soundness conditions immediately:
\begin{proposition}\label{soundMultiagent}
    For every maximal consistent set $\Gamma$ and $i\in \BA$,
    $F_i(\Gamma)$ contains all projection functions on $\BG$ and
    is closed under composition.
\end{proposition}
\begin{proof}
    Take a projection function $f(x_1, \ldots x_n) = x_k$.
    Then by the definition of join in a lattice, 
    $$\hs(x_k) \le \bigvee \llrr{\hs(x_1), \cdots ,\hs(x_n)}$$
    since $\hs(x_k) \in \llrr{\hs(x_1), \cdots ,\hs(x_n)}$.

    For function composition,
    let $\ol{x}$ represent a sequence of variables and
    $\hs(\ol{x})$ the sequence after the application of $\hs$.
    Then take a function
    $f(\ol{x}) = g_0(g_1(\ol{x}_1), \cdots g_n(\ol{x}_n))$ where
    $\ol{x}$ includes the union of all $\ol{x}_k$s and
    all $g$ functions are already in $F_i({\Gamma})$.
    Now
    \[
        \begin{array}{rcl}
          \hs(f(\ol{x})) &  =& \hs(g_0(g_1(\ol{x}_1), \cdots g_n(\ol{x}_n))) \\
                         &\le& \bigvee\llrr{\hs(g_1(\ol{x}_1), \cdots g_n(\ol{x}_n))}\\
                         &\le& \bigvee\llrr{\lor\hs(\ol{x}_1), \cdots \lor\hs(\ol{x}_n)}\\
                         &\le& \bigvee\hs(\ol{x}).
        \end{array}
    \]
    This shows that the composition $f$ satisfies the requirement and is in $F_i(\Gamma)$.
\end{proof}

The next proposition shows why
we use the dependence lattice to define the function domains for each agent.
The proposition says that to make $\Kf_i(C, d)$ true,
we only need to make sure that functionality holds,
and to make $\Kf_i(C, d)$ false,
we do not need to pay any special attention as
the function domain $F_i(\Gamma)$ has already taken care of everything. 
\begin{proposition}\label{kfproof}
    For every $\sigma \in 2^{\BQ}$,
    define $v_{\sigma}: \BQ \to \BG, d \mapsto \lr{d, \sigma(d)}$.
    This means we restrict the value of $d \in \BQ$ to be
    $\lr{d, 0}$ or $ \lr{d, 1}$.
    Now for every maximal consistent set $\Gamma$, $i\in\BA$,
    $C\subseteq_{fin} \BQ, d\in \BQ$, and $\Sigma \subseteq 2^{\BQ}$:
    \begin{itemize}
        \item if $\Sigma$ satisfies the functionality condition for $C, d$,
        namely for all $\sigma_1, \sigma_2\in \Sigma$,
        $\sigma_1(C) = \sigma_2(C)$ implies $\sigma_1(d) = \sigma_2(d)$,
        and if $\Kf_i(C, d) \in \Gamma$,
        then there exists $f\in F_i(\Gamma)$ such that
        for all $\sigma\in \Sigma$, $v_{\sigma}(d) = f(v_{\sigma}(C))$;
        \item if $\Kf_i(C, d) \not\in \Gamma$ then
        for all  $\sigma\in \Sigma$ and for all $f\in F_i(\Gamma)$,
        $v_{\sigma}(d) \not= f(v_{\sigma}(C))$.
    \end{itemize}
\end{proposition}
\begin{proof}
    First notice that in the definition of $F_i(\Gamma)$,
    the restriction actually forgets the second coordinate of the inputs and outputs.
    But it is the second coordinate that all $\sigma\in \Sigma$ try to adjust.
    By definition,
    the first coordinates of $v_{\sigma}(c)$ for all $c\in \BQ$ are just themselves.
    So for all $c\in \BQ, \sigma\in\Sigma$,
    $\hs(v_{\sigma}(c)) = \cl(c)$.
    
    If $\Kf_i(C, d)\in \Gamma$,
    then (dropping the super and subscripts) $d\in\cl(C)$.
    This means the same as $\llrr{d} \subseteq \cl(C)$, which,
    by the fact that $\cl$ is a closure operator,
    implies $\cl(d) \subseteq \cl(\cl(C)) = \cl(C)$.
    Then $\cl(d)\subseteq \cl(C)$,
    which means $\hs(v_{\sigma}(d)) \le \cl(C)$ in $\lat$ for all $\sigma \in \Sigma$.
    Also,
    $\cl(C) = \bigvee \llrr{\cl(c_1), \cl(c_2), \ldots \cl(c_n)} = \bigvee\hs(v_{\sigma}(C))$ for all $\sigma\in \Sigma$.
    So indeed $\hs(v_{\sigma}(d)) \le \bigvee\hs(v_{\sigma}(C))$ in $\lat$.
    Together with the functionality assumed for $\Sigma$,
    this means mapping $v_{\sigma}(C)$ to $v_{\sigma}(d)$ simultaneously for all $\sigma\in \Sigma$
    is allowed in $F_i(\Gamma)$.
    Then we can extend this partial map to
    a map from $\BG^n$ to $\BG$ in $F_i(\Gamma)$.
    An easy solution is to do projection for all other possible inputs.

    If $\Kf_i(C, d)\not\in \Gamma$,
    then $d\not\in\cl(C)$ and hence $\cl(d) \not\subseteq \cl(C)$.
    If $\Sigma$ is empty,
    the statement is trivially true.
    So assume $\Sigma$ is not empty.
    Now take an arbitrary $\sigma\in \Sigma$.
    Then $\hs(v_{\sigma}(d)) \not\le \bigvee\hs(v_{\sigma}(C))$,
    which violates the restriction on $F_i(\Gamma)$ if $v_{\sigma}(C)$ is to be mapped to $v_{\sigma}(d)$.
    Thus, for all $f\in F_i(\Gamma), v_{\sigma}(d) \not= f(v_{\sigma}(C))$. \qed 
    
\end{proof}

This proposition says that
the dependency lattice $\lat_i^{\Gamma}$ and
the corresponding function domain $F_i(\Gamma)$
form a suitable representation of the function domain that
$i$ uses implicitly given $i$'s knowledge and ignorance in $\Gamma$.
As we hinted before the construction,
this function domain is so specific about what is possible that
when $\Kf_i(C, d)$ is not known,
it is not rejected by a failure of functionality,
which requires at least two epistemically possible assignment,
but by a failure of conforming to the prior knowledge encoded in the function domain,
as shown by the second bullet in the previous proposition.
On the other hand,
once functionality holds in all possible assignments,
we do not need to worry about whether the function domain allows it or not,
which is clear from the proof of the first bullet.
Thus, this $F_i(\Gamma)$ is a perfect choice.

For the $\Kv_i$ part,
we need to adjust the assignments of variables
to construct more (epistemically) possible assignments
to reject formulas like $\Kv_i(d)$ which is not in $\Gamma$:
if in one world $d$ is assigned to be $x$,
then we want to make an adjustment to get a new world where
it is assigned to $y\not=x$.
This will be done by moving the value of $d$ to $\lr{d, 1}$ from $\lr{d, 0}$ or vice versa.
And for agent $i$ in a maximal consistent set $\Gamma$,
the variables to be moved are exactly
$\ol{\kvig} = \llrr{d \sep \Kv_i(d) \not\in \Gamma}$,
the complement of the set of the variables with a known value by $i$.
By maximality,
it is also the collection of all $d\in \BQ$ such that
$\lnot \Kv_i(d)\in \Gamma $.
It is crucial to move the value of all variables in $\ol{Kv_{i, \Gamma}}$ at once,
as otherwise there might be some unwanted violation of functionality:
even though for both $\sigma = \sigma_1, \sigma_2$,
$\hs(v_{\sigma}(d)) \le \bigvee \hs(v_{\sigma}(C))$,
it could be that $v_{\sigma_1}(C) = v_{\sigma_2}(C)$ while $v_{\sigma_1}(d) \not= v_{\sigma_1}(d)$.
So in this case, no functional dependency exists from $C$ to $d$,
but the reason is not that $d$ is at the wrong place in the lattice,
but instead the failure of functionality.
We must avoid this situation,
by changing all values of variables in $\ol{\kvig}$ simultaneously
when producing a new possible assignments in a new possible world.
This motivates the following definition:

\newcommand{\mvig}{\mv_i^{\Gamma }}
\newcommand{\mvigp}{\mv_i^{\Gamma' }}
\newcommand{\mvigw}{\mv_i^{\Gamma_1 }}
\newcommand{\mvigt}{\mv_i^{\Gamma_2 }}

\begin{definition}[Value Move] \label{valueMove}
    Given $\Gamma$ a maximal consistent set and $i\in \BA$,
    define the value move operator $\mvig:2^{\BQ} \to 2^{\BQ}$:
    $$\mvig(\sigma)(d) = \left\{
        \begin{array}{ll}
          \sigma(d)      & d\in \kvig \\
          1 - \sigma(d)  & d\in \ol{\Kv_{i,\Gamma}}.
        \end{array}
    \right.$$
\end{definition}
This operator captures agent $i$'s switching of the values of the variables in $\ol{\kvig}$ all at once.
Two important properties should be noted.
First, $\mvig$ is dependent only on $\kig$.
Indeed it only depends on $\kvig$ but
because of the axioms $\ttt{KV4}$ and $\ttt{KV5}$,
it is equivalent to say that it depends only on $\kig$.
This means that if $\kig = \kigp$, then as an operator, $\mvig = \mvigp$.

Another important property of this operator is that
$\mvig(\mvig(\sigma)) = \sigma$ for all $\Gamma, i, \sigma$
ranging over maximal consistent sets, $\BA$ and $2^{\BQ}$.
Thus, it is actually an inverse operator. 

\newcommand{\ww}{\lr{\Gamma, \sigma}}
\newcommand{\wwp}{\lr{\Gamma', \sigma'}}
\newcommand{\www}{\lr{\Gamma_1, \sigma_1}}
\newcommand{\wwt}{\lr{\Gamma_2, \sigma_2}}

Equipped with the above definitions, the canonical model can now be defined:
\begin{definition}[Canonical Model]
    Build a model $\modelM = \lr{W, \lr{\sim_i}_{i\in\BA}, U, V, \lr{\BF_i}_{i\in\BA}}$ as follows:
    \begin{itemize}
        \item $W = \llrr{\lr{\Gamma, \sigma} \sep \Gamma \textrm{\ a\ maximal\ consistent\ set}, \sigma\in 2^{\BQ}},$
        \item $\ww \sim_i \wwp$ iff
        \begin{enumerate}
            \item $\kig = \kigp$,
            which says that two worlds must share the same set of knowledge of $i$, and
            \item $\sigma = \sigma'$ or $\sigma = \mvig(\sigma')$,
            which says that any agent $i$ needs to see some different possible assignments of the variables,
            but not too many: just two,
        \end{enumerate}
        \item $U(\ww, p) = [p\in \Gamma],$
        \item $V(\ww, d) = \lr{d, \sigma(d)},$
        or equivalently using notations introduced above in Proposition \ref{kfproof}, $V(\ww) = v_{\sigma},$
        \item $\BF_i(\ww) = F_i(\Gamma).$
    \end{itemize}
\end{definition}
Before proving the truth lemma,
it must be shown that $\modelM$ is indeed a model of $\logicLKVF^*_m$.
This amounts to checking the following:
\begin{itemize}
    \item $\sim_i$ is an equivalence relation for all $i\in\BA$,
    \item $\BF_i(\ww)$ satisfies the soundness condition,
    \item if $\ww \sim_i \wwp$ then $\BF_i(\ww) = \BF_i(\wwp)$.
\end{itemize}
Because $\sim_i$ is defined using equality,
its reflexivity is easy to see.
We need the two special properties of $\mvig$
noted right after the Definition \ref{valueMove}
to show symmetry and transitivity.

For symmetry, suppose $\ww \sim_i \wwp$.
Then $\kig = \kigp$.
Thus, $\mvig = \mvigp$ and $\sigma = \mvig(\sigma') = \mvigp(\sigma')$.
Also, as $\mvigp$ is an inverse operator,
by applying it twice,
we get $\sigma' = \mvigp(\sigma)$.
So it can be concluded that $\wwp \sim_i \ww$.

Transitivity can be shown similarly.
Suppose $\www \sim_i \ww \sim_i \wwt$.
It immediately follows that $\mvigw = \mvig = \mvigt$.
So we can treat all of them as $\mvig$.
Then we know
    $\sigma = \sigma_1$ or $\sigma = \mvig(\sigma_1)$, and
    $\sigma = \sigma_2$ or $\sigma = \mvig(\sigma_2)$.
There are in total four possibilities
depending on which disjuncts hold,
and the only less trivial one is when
$\sigma_1 = \mvig(\sigma)$ and $\sigma = \mvig(\sigma_2)$.
But if that is the case, then $\sigma_1 = \mvig(\mvig(\sigma_2)) = \sigma_2$.
So transitivity holds.

The soundness condition was already shown
when $\BF_i(\ww) = F_i(\Gamma)$ is defined in
Proposition \ref{soundMultiagent}.
We also noted that $F_i(\Gamma)$ only depends on $\kig$
because it only depends on the dependency lattice $\lat_i^{\Gamma}$,
which in turn only depends on $\kig$.
If $\ww \sim_i \wwp $, $\kig = \kigp$ and $F_i(\Gamma) = F_i(\Gamma')$,
so indeed $\BF_i(\ww) = \BF_i(\wwp)$.
So we conclude that $\modelM$ is a model of $\logicLKVF^*_m$.

The unconventional second condition for $\sim_i$ is there
for the purpose of preventing unwanted failure of functionality.
As explained after Proposition \ref{kfproof},
we are not refuting $\Kf(C, d)$ using functionality,
so it is better to keep the functionalities between as many variables as possible.
In particular,
all functionalities between the variables in $\ol{Kv_{i, \Gamma}}$ can be preserved.
The condition does this by requiring that
if $i$ sees more than one possibility for some variables,
then all the values of $\ol{Kv_{i, \Gamma}}$ must change
to a different epistemic possibility together using the value move operator.
This makes impossible the situation where
one variable in $\ol{Kv_{i, \Gamma}}$ realizes a different possibility while
another stays the same, a situation that characterizes the failure of functionality. 

Now the truth lemma in this case can be proven:
\begin{lemma}[Truth Lemma]
    For all $\phi$ in the language of $\logicLKVF_m^*$ and
    all maximal consistent sets $\Gamma$ in the axiom system $\axiomLKVF_m$,
    $\modelM, \ww \vDash \phi$ if and only if $\phi \in \Gamma$.        
\end{lemma}
\begin{proof}
    Use induction on $\phi$.
    The propositional letters and boolean combination cases are conventional.
    We focus on the knowledge cases.

    $\phi = K_i\psi$.
    If $K_i\psi\in\Gamma$, then by the definition of $\sim_i$,
    for all $\wwp\sim_i\ww$, $\kig = \kigp$.
    Thus, $\psi\in \kigp$ and $K_i\psi \in \Gamma'$.
    By axiom $\ttt{T}$,
    $\psi\in\Gamma'$, and using the induction hypothesis,
    $\modelM, \wwp \vDash \psi$.
    Thus, $\modelM, \ww\vDash K_i\psi$ by the semantic clause of $K_i$.

    If $K_i\psi\not\in\Gamma$,
    then by a standard argument using axioms and the maximality of $\Gamma$,
    $\kig \cup \llrr{\lnot\psi}$ is consistent and expandable to
    a maximal consistent set $\Gamma'$.
    Then $\lr{\Gamma', \sigma} \sim_i \ww$ and
    $\modelM, \lr{\Gamma, \sigma} \vDash \lnot\psi$ by the induction hypothesis.
    So $\modelM, \ww\not\vDash K_i\psi$.

    $\phi = \Kv_i(d)$.
    If $\Kv_i(d) \in \Gamma$,
    then $d\in \kvig$ and thus $\mvig(\sigma)(d) = \sigma(d)$.
    Now for all $\wwp \sim_i \ww$,
    $\sigma$ is equal to $\sigma'$ or $\mvig(\sigma')$.
    But as $d\in \kvig$,
    $\mvig$ is not changing the value of $d$.
    So in either case, $\sigma'(d) = \sigma(d)$.
    Thus, the value of $d$ is fixed to $\lr{d, \sigma(d)}$
    among all worlds accessible by $i$ from $\ww$.

    If $\Kv_i(d) \not\in \Gamma$,
    then $d\not\in \kvig$
    and $\mvig(\sigma)$ will change the value of $d$.
    Take the world $\wwp$ with $\sigma' = \mvig(\sigma)$.
    Then $\sigma = \mvig(\sigma')$, so $\ww \sim_i \wwp$.
    Also, $\sigma'(d) = 1 - \sigma(d) \not= \sigma(d)$.
    Thus, $V(\wwp, d) \not= V(\ww, d)$.
    By the semantic clause of $\Kv_i(d), \modelM, \ww \not\vDash \Kv_i(d).$

    $\phi = \Kf_i(C, d)$.
    Suppose $\Kf_i(C, d)\in \Gamma$.
    Then we should first show that the functionality condition holds.
    For any $\www, \wwt \sim_i \ww$, if $V(\www, C) = V(\wwt, C)$,
    then there are two possibilities
    \begin{itemize}
        \item $C\subseteq \kvig$.
        Then by axiom $\ttt{VF}$, $d\in \kvig$ as well,
        and by the argument in the previous case,
        $V(\www, d) = V(\wwt, d) = \lr{d, \sigma(d)}$.
        
        \item $C\not\subseteq \kvig$.
        Then take $c\in C \cap \ol{\kvig}$.
        Since $V(\www, C) = V(\wwt, C)$,
        $\sigma_1(c) = \sigma_2(c)$.
        Because $\www \sim_i \ww \sim_i \wwt$,
        $\www \sim_i \wwt$.
        So either $\sigma_1 = \sigma_2$ or $\sigma_1 = \mvigw(\sigma_2)$.
        But the latter case cannot happen because
        if that is true, then $\sigma_1(c) \not= \sigma_2(c)$ since $c\in \ol{\kvig}$.
        So $\sigma_1 = \sigma_2$ and in particular $\sigma_1(d) = \sigma_2(d)$.
        Thus, $V(\www, d) = V(\wwt, d)$.
    \end{itemize}
    Indeed, by our definition of $\sim_i$,
    among all worlds accessible from $\ww$ by $\sim_i$,
    there are altogether only two possible valuations:
    $\sigma$ and $\mvig(\sigma)$.
    Thus, by applying Proposition \ref{kfproof} to set $\Sigma = \llrr{\sigma' \sep \wwp \sim_i \ww}$,
    it follows that there exists a function $f\in F_i(\Gamma) = \BF_i(\ww)$ such that
    $V(\wwp, d) = f(V(\wwp, C))$ for all $\wwp\sim_i \ww$.
    So $\modelM, \ww \vDash \Kf_i(C, d)$.

    If $\Kf_i(C, d) \not\in \Gamma$,
    then by Proposition \ref{kfproof} again,
    for every function $f\in F_i(\Gamma) = \BF_i(\ww)$,
    there exists $\wwp \sim_i \ww$ such that $V(\wwp, d) \not= f(V(\wwp, C))$.
    Actually $\ww$ itself works here.
    Thus, $\modelM, \ww \not\vDash \Kf_i(C, d)$.
\end{proof}

From the truth lemma,
it can be concluded that every consistent set is satisfied somewhere in the canonical model $\modelM$ built above.
So the completeness of $\axiomLKVF$ with respect to $\logicLKVF^*_m$ follows.
Together with the soundness proven in Proposition \ref{soundMultiagent}, we obtain an axiomatization of $\logicLKVF^*_m$:
\begin{theorem}
    Under the cardinality requirement $\BG\ge |\BQ \times \llrr{0, 1}|$, $\axiomLKVF_m$ is an axiomatization of $\logicLKVF_m^*$.
\end{theorem}

\section{Discussion and Future Work}
First, we discuss the semantics of the $\Kf$ operator.
Obviously, while $\Kv(d)$ means that
there is only one value for $d$ to take,
in general,
the truth of $\Kf(C,d)$ does not force
the set of possible functional dependency relations of $d$ on $C$
to be a singleton.

It could be argued that
the agent can nevertheless regard all those candidates as equivalent,
because they must have exactly the same behavior over the partial domain
$P = \llrr{V(w, C) \sep w\in W}$.
And things in $\BG^{|C|}$ but outside this set $P$ are epistemically impossible.
Thus, the behavior of functions on $\BG^{|C|}\backslash P$ is something that
our agent can and will ignore
if situations epistemically impossible do not concern the agent.
One example, also mentioned in the introduction,
is when ``knowing-value'' is the real objective of the agent and
``knowing-dependency'' only expresses the agent's potential to know more values.
The semantics proposed in this paper allows adjustments to $\BF$,
which might be a consequence of an agent's concern about situations epistemically impossible,
but not necessarily.
And even if it is the case,
the semantics does not show
how $\BF$ is derived from what concerns of the agents. 

It is not uncommon that
epistemic possibilities are not the right place to stop
when evaluating knowledge of functional dependency.
Consider the following example:
\begin{quotation}
    \noindent I know the color of my hair.
    Therefore, I know the color of my hair functionally depends on the number of fingers I have.
\end{quotation}
This argument is very hard to swallow intuitively.
Yet it is validated by the axiom $\mathtt{EXT}$.
Indeed, in the current setting of the semantics of $\Kf $,
to validate this,
we only need to allow a moderate amount of constant functions in our function domain.
The root of the problem is that,
in a pure epistemic logic setting,
if something is known,
the agent has no access to other alternatives as knowledge is the only modality here,
whereas in most realistic situations,
even when something is known,
we have modal access to some possibilities
different from the known one.
For example,
possibilities in the future or past
can be used to explain why the color of my hair is not really dependent on the number of fingers I have.
And even when I have not and will not change the color of my hair,
we can still use metaphysical possibilities:
``the color of my hair \textit{could} be different, regardless of how many fingers I have.''

\newcommand{\Kfb}{\Kf^{\Box}}

Thus, it might be of interest to
capture knowledge of functional dependency in another modality.
To do this we can add a new modality $\Box$ interpreted by a relation $R$.
Then ``knowing a/the functional dependency'' can now be expressed
by an operator $\Kfb$ with the following semantics:
\begin{multline*}
    \lr{W, \sim, R, U, V}, w \vDash \Kfb(C, d) \Leftrightarrow
    \exists f\in \BF, \forall w'\sim w, \forall w^*, Rw'w^* \Rightarrow V(w^*, d) = f(V(w^*, C))
\end{multline*}
where $\sim $ is the epistemic indistinguishability relation.
This definition still says that
there exists a function that works for all epistemically indistinguishable worlds.
But here ``works'' means $f$ captures the functional dependency of $d$ upon $C$
with respect to another modality $\Box $ which might be different from $K $.

The choice of $R$ can be arbitrary,
but at least two interesting candidates are immediate:
an equivalence relation to capture metaphysical possibilities and 
a linear or branching time relation used in temporal logics.
A simple observation is that,
if we still want a new version of $\ttt{VF}$, namely
$$\ttt{VF'}: \bigwedge_{c\in C}\Kv(c) \land \Kfb(C, d) \to \Kv(d)$$
to be valid,
we need $R$ to be reflexive.
Otherwise, the functional dependency might be only talking about worlds far away from the actual world,
though accessible through $R$.
Since the choice for $R $ can be flexible,
there will be many interesting results to be discovered under this semantics.
In particular, for the study of completeness,
we might want to add more first order features to facilitate a proof
more similar to its first order counterpart,
a strategy successfully employed in \cite{baltag16:edl}.
It might be desirable because, with two modalities,
a direct construction of value assignments can be unmanageable.

But a demanding reader may still not be satisfied,
as even if we add a new modality,
the choice of the functions could be nonunique again.
This motivates another interpretation of knowledge of functional dependency,
emphasizing even more the ``knowledge'' part:
$\Kf(C, d)$ says that the agent has gathered so much information
that there is (almost) exactly one function
that can be used to explain the data he/she has seen so far.
Thus, knowledge appears only when there is only one possible or a few very plausible explanations.
If there is no possible explanation in the sense that no function in the function domain $\BF$ is applicable,
or there are too many explanations,
no knowledge is obtained.
This sounds natural,
but much more technically will be needed to formalize this:
either a counting operator,
or a probabilistic operator tracking the posterior distribution over the candidate explanations.  

There are also interesting possible extensions of the framework given in this paper.
For example,
the multiagent case here assumed a no-interaction semantics.
But once we require prior knowledge of possible functions
to be available to other agents,
interesting interactions will appear.
For example,
suppose $\BF_j$ is known to agent $i$,
i.e., if $w \sim_i w'$ then $\BF_j(w) = \BF_j(w')$.
Then the following is valid:        
$$\Kv_i(c) \land \Kv_i(d) \land K_i(\Kv_j(c) \land \Kv_j(d)) \to K_i\Kf_j(c, d) \lor K_i\lnot\Kf_j(c, d).$$
Intuitively this says that
if agent $i$ knows the values of $c, d$ and
knows that agent $j$ knows,
then either $i$ knows that $j$ has an explanation of the value of $c,d$ or $i$ knows that $j$ does not have one.
The antecedent fixes the value of $c, d$ in all worlds accessible first from $i$ and then from $j$.
Thus if $j$ fails or succeeds to explain this particular instance,
agent $i$ knows it.
Stronger interactions will appear
if we require all agents
to share a single prior knowledge base $\BF$,
i.e., for all $i, w$, $\BF_i(w) = \BF$.
Then the following is valid:
$$\Kv_i(c) \land \Kv_i(d) \land K_i(\Kv_j(c) \land \Kv_j(d)) \to (\Kf_i(c, d) \to \Kf_j(c, d)).$$
This says that if $i$ knows the value of $c,d$ and knows that $j$ knows them,
then $i$ being able to explain this instance implies that $j$ can explain it as well.
To axiomatize these two cases, new axioms and techniques will emerge.
Further, we can also add an operator that expresses knowledge about other agents' function domain. 

Computationally, we see without too much surprise that
the finite model property holds.
For all the three single agent cases with a finite language,
the required size of $\BG$ and the size of the model constructed can be explicitly computed.
In the multiagent case,
a standard filtration method can also be applied quite straightforwardly.
Notice that in each of the three cases,
the completeness proof requires a minimal size of $\BG$.
A natural question is whether we can bring down the size requirement by giving more economic completeness proofs.
In particular, the double exponential size requirement in the single agent fixed intermediate function domain case
seems to be too large,
while the number of value assignments seems too small (just 2).
We might be able to implement a trade-off here or a smarter lattice construction.

In summary,
introducing knowledge about functional dependency relations
brings us ample new opportunities to extend the border of epistemic logic.
There will be a lot more to achieve.

\bibliographystyle{plain}
\bibliography{sgwyjj}
\end{document}